\documentclass[conference,letterpaper,10pt,twocolumn]{ieeeconf}
\IEEEoverridecommandlockouts                              % This command is only needed if
\bstctlcite{IEEEexample:BSTcontrol}

\usepackage{scalefnt}
\usepackage{setspace}
\usepackage{romannum}
\usepackage[usenames,dvipsnames]{xcolor}
\usepackage{tkz-berge}
\usetikzlibrary{fit,shapes}
\usepackage{calrsfs}
\usepackage{graphicx}
\usepackage{graphics}
\usepackage{epsfig}
\usepackage{mathptmx}
\usepackage{caption}
\usepackage{subcaption}
\usepackage{times}
\usepackage{tikz}
\usetikzlibrary{positioning}
\usepackage{amsmath, amssymb}
\usetikzlibrary{arrows}
 \usepackage[T1]{fontenc}
\usepackage{bm, bbm}

\usepackage{graphicx}
\usepackage{subcaption}

\usepackage{kantlipsum}
\usepackage{amsthm}
\usepackage{amsfonts}
\usepackage{setspace}
\usepackage{multirow}
\usepackage{textcomp}
\usepackage[english]{babel}
\usepackage{float}
\usepackage{algorithmicx}
\usepackage[section]{algorithm}
\usepackage{algpascal}
\usepackage{algc}
\usepackage{algcompatible}
\usepackage{multicol}
\usepackage{algpseudocode}
\let\oldReturn\Return
\renewcommand{\Return}{\State\oldReturn}
\usepackage{mathtools}
\usepackage{bm}
\usepackage{mathptmx}
\usepackage{times}
\usepackage{mathtools, cuted}
\usepackage{textcomp}
\usepackage[english]{babel}
\usepackage{rotating}
\usepackage{pgfplots}
\pgfplotsset{compat=1.5}
\usepackage{siunitx}
\usepackage{pgfplotstable}
\usepackage{filecontents}

\usepackage{bbm}
\usepackage{bbold}
\usepackage{float}
\renewcommand{\baselinestretch}{0.97}  
\def\BibTeX{{\rm B\kern-.05em{\sc i\kern-.025em b}\kern-.08em
    T\kern-.1667em\lower.7ex\hbox{E}\kern-.125emX}}

\newtheorem{assume}{Assumption}
\newtheorem{lemma}{Lemma}
\newtheorem{proposition}{Proposition}
\newtheorem{theorem}{Theorem}

\DeclareMathAlphabet{\pazocal}{OMS}{zplm}{m}{n}

\newcommand{\norm}[1]{\left\lVert#1\right\rVert}
\newcommand{\X}{\pazocal{X}}

\newcommand{\bR}{\mathbb{R}}
\newcommand{\R}{\pazocal{R}}
\renewcommand{\bR}{\mathbb{R}}
\newcommand{\xs}{x^{\star}}
\newcommand{\C}{{\pazocal{ C}}}
\newcommand{\D}{{\pazocal{ D}}}
\newcommand{\B}{{\pazocal{ B}}}
\renewcommand{\P}{{\pazocal{ P}}}

\newcommand{\pxc}{\phi_{x_t, c_t}}
\newcommand{\bp}{\bar{\psi}}

\makeatletter
\renewcommand\subsection{\vspace*{-0.5 mm}\@startsection{subsection}{2}{\z@}%
                                      {-3.25ex\@plus -1ex \@minus -.2ex}%
                                      {-1ex \@plus .2ex}%
                                      {\it{\normalfont}}}
 \makeatother

                     \makeatletter
\makeatletter
\newcommand{\specificthanks}[1]{\@fnsymbol{#1}}% Inserts a specific \thanks symbol
\makeatother
\title{\LARGE \bf Stochastic Conservative Contextual Linear Bandits}
\author{Jiabin Lin, Xian Yeow Lee, Talukder Jubery, Shana Moothedath, \\Soumik Sarkar,  and Baskar Ganapathysubramanian
\thanks{J. Lin and S. Moothedath are with the Department of Electrical and Computer Engineering, Iowa State University, USA. Email: $\lbrace$jiabin, mshana $\rbrace$@iastate.edu.}
\thanks{X. Y. Lee, T. Jubery, S. Sarkar, and B. Ganapathysubramanian are with the 
Department of Mechanical Engineering, Iowa State University, USA. Email: $\lbrace$xylee, znjubery,soumiks, baskarg$\rbrace$@iastate.edu.}
}

\begin{document}
\maketitle

\begin{abstract}
Many physical systems have underlying safety considerations that require that the strategy deployed ensures the satisfaction of a set of constraints.  Further, often we have only partial information on the state of the system.  We study the problem of safe real-time decision making under uncertainty.  In this paper, we formulate a conservative stochastic contextual bandit formulation for real-time decision making when an adversary chooses a distribution on the set of possible contexts and the learner is subject to certain safety/performance constraints. The learner  observes only the context distribution and the exact context is unknown, and the goal is to develop an algorithm  that selects a sequence of optimal actions to maximize the cumulative reward without violating the safety constraints at any time step. 
By leveraging the UCB algorithm for this setting, we propose a  conservative linear UCB algorithm for stochastic bandits with context distribution. We prove an upper bound on the regret of the algorithm and show that it can be decomposed into three terms: (i)~an upper bound for the regret of the standard linear UCB algorithm, (ii)~a constant term (independent of time horizon)  that accounts for the loss of being conservative in order to satisfy the safety constraint, and (ii)~a constant term (independent of time horizon) that accounts for the loss for the contexts being unknown and only the distrbution being known. 
% We propose a conservative stochastic Multi-Arm Bandit (MAB) formulation to adress this problem. Our bandit formulation is a conservative bandit setting since we incorporate constraints on the learned policy such that the learned policy  need to satisfy certain baseline performance criteria while maximizing the reward. Furthermore, our bandit formulation is  stochastic in the sense that the contexts are not observable, rather a distribution of the contexts are known.  
  To validate the performance of our approach we perform extensive simulations on synthetic data and on real-world  maize data collected through the Genomes to Fields (G2F) initiative.  
\end{abstract}
%%%%%%%%%%%%%%%%%%%%%%%%%%%%%%%%%%%%%%%%%%%%%%%%%%
\section{Introduction}
Decision making under critical and uncertain situations is a common problem in a wide range of domains including online marketing, finance, health sciences, and robotics.  There exist  learning algorithms that can learn good policies/strategies for optimal decision making.
Contextual bandits is one such framework that models the sequential  decision making process by utilizing the side information which is referred to as context \cite{bubeck2012regret}. 
%Contextual bandits arises in many applications in science and engineering. 
One real-world example of a contextual multi-armed bandit problem is when a news website has to make a decision about which articles to display to a visitor when some information about the visitor is known \cite{li2010contextual}.
In the contextual bandit model a learner interacts with the environment in several rounds. In each round  the environment presents a context to the learner and the goal of the learner is to choose an action. Upon selecting an action the learner is presented with a reward associated with the chosen action and the goal of the learner is to maximize the cumulative reward.  

Most of the existing work on contextual bandit model assumes that the contexts are known and there are no additional constraints on the learner. However, in many applications there exists scenarios where the contexts are noisy or are forecasting measurements (e.g., weather forecasting or stock market prediction) so that the actual context is unknown, rather a distribution on the context is only available. In such cases, the exact context is a sample from this distribution. Such a model has been studied in \cite{kirschner2019stochastic} and an Upper Confidence Bound (UCB)-based algorithm has been proposed with regret bound guarantee. Additionally, safety/performance is  a major concern while making decisions and it is crucial to develop learning algorithms that can perform decision making while ensuring that certain safety/performance conditions are satisfied at each round. Contextual bandits with safety constraints have been studied in \cite{kazerouni2016conservative, wu2016conservative, amani2019linear} and algorithms with guarantees were proposed. 

Our goal in this paper is to develop a framework for solving sequential decision making when the contexts are unknown and there are safety/performance constraints imposed on the learner. We motivate our problem setting through a scenario. Consider a scenario where the goal is to develop a recommendation system
for smart farming such that based on the details of the
farm and the farming conditions, including information on
the weather and soil properties, the system presents recommendations
on the choice of the crop/seed in order to
maximize the overall net profit of the farmer. In this setting, the contexts are not observable,
rather a distribution of the contexts are known as the weather and soil conditions are forecasting rather than accurate measurements. Additionally,
often farmers impose performance constraints such
as the net profit must be at least a certain value. Thus for a given
farmland and set of soil properties and climate indices, our
goal is to provide recommendations for the crop/seed type
such that the annual net profit of the farmer is maximized
and the associated constraints are satisfied.

This paper makes the following contributions.

\begin{enumerate}
\item[$\bullet$] We formulate a conservative stochastic contextual bandit formulation for real-time decision making when an adversary chooses a distribution on the set of possible contexts and the learner is subject to certain safety/performance constraints.

\item[$\bullet$] We present a UCB-based algorithm, conservative (safe) linear UCB algorithm for stochastic bandits with context distribution and unknown contexts.

\item[$\bullet$] We prove an upper bound on the regret of the algorithm and show that it can be decomposed into three terms: (i)~an upper bound for the regret of the standard linear UCB algorithm, (ii)~a constant term (independent of time horizon)  that accounts for the loss of being conservative in order to satisfy the safety constraint, and (ii)~a constant term (independent of time horizon) that accounts for the loss for the contexts being unknown and only the distrbution being known. 

\item[$\bullet$] We validated the performance of our approach via extensive simulations on synthetic data and on real-world  maize data collected through the Genomes to Fields (G2F) initiative.  
\end{enumerate}

   The rest of the paper is organized as follows. In Section~\ref{sec:not} we present the notations and the problem formulation. In Section~\ref{sec:rel}, we present the related work. In Section~\ref{sec:sol} we present the solution approach for the conservative stochastic bandit problem. In Section~\ref{sec:reg}, we present the regret analysis and prove an upper bound on the regret of our proposed algorithm. In Section~\ref{sec:con}, we present the conclusion and future work.

\section{Notations and Problem Formulation}\label{sec:not}
In this section, we first specify the standard  linear bandit problem below and then explain the stochastic constrained bandit setting.
Let $\X$ denote the action set and $\C$ denote the context set. The environment is defined by a fixed and unknown  function $y: \X \times \C \rightarrow \mathbb{R}$.
 In linear bandit setting, at any time $t \in \mathbb{N}$, the agent observes a context $c_t \in \C$ and  has to choose an action $x_t \in \X$. Each context-action pair $(x,c)$, $x \in \X$ and $c \in \C$, is associated with a feature vector $\phi_{x,c} \in  \mathbb{R}^d$, i.e., $\pxc = \phi(x_t, c_t)$. Upon selection of an action $x_t$, the agent observes a  reward $y_t \in  \mathbb{R}$
\begin{equation}
y_t :=  \langle\theta^\star, \phi_{x_t, c_t}  \rangle + \eta_t,\label{eq:reward}
\end{equation}
where $\theta^\star \in \mathbb{R}^d$ is the unknown reward parameter, $ \langle\theta^\star, \phi_{x_t, c_t}  \rangle  = r (x_t, c_t)$ is the expected reward for action $x_t$ at time $t$, i.e., $r(x_t, c_t) = \mathbb{E}[y_t]$, and $\eta_t$ is  $\sigma-$Gaussian, additive  noise. The goal is to choose optimal actions $\xs_t$ for all $t \in T$ such that  the cumulative reward, $\sum_{t=1}^T y_t$, is maximized. This
is equivalent to minimizing the cumulative (pseudo)-regret denoted as 
 \begin{equation}
 \R_T = \sum_{t=1}^T\langle\theta^\star, \phi_{\xs_t, c_t}^t  \rangle - \sum_{t=1}^T\langle\theta^\star, \phi_{x_t, c_t}^t  \rangle.\label{eq:regret}
 \end{equation}
 Here $\xs_t$ is the optimal/best action for context $c_t$ and $x_t$ is the action chosen by the agent for context $c_t$. We make the standard assumptions on the additive noise $\eta_t$ and the unknown parameter $\theta^\star$ \cite{kirschner2019stochastic, kazerouni2016conservative}.
 
%  Assumptions~\ref{assume:noise} and~\ref{assume:bound} are presented below.
 \begin{assume}\label{assume:noise}
Each element $\eta_t$ of the noise sequence $\{\eta_t\}_{t=1}^{\infty}$ is conditionally $\sigma-$subGaussian, i.e., 

\begin{eqnarray*}
\mbox{For~all~} \zeta \in \R, \mathbb{E}[e^{\zeta \eta_t}|x_{1:t, \epsilon_{1:t-1}}] \geqslant exp(\dfrac{\zeta^2 \sigma^2}{2}).
\end{eqnarray*}
\end{assume}

\begin{assume}\label{assume:bound}
There exists constant $A, D \geqslant 0$ such that $\norm{\theta^\star}_2 \leqslant A$, $\norm{\phi_{x,c_t}}_2 \leqslant D$, and $\phi_{x, c_t}^{\top}\theta \in [0,1]$, for all $t$ and all $x \in \X$.
\end{assume}

 In this work, we consider a {\em conservative and stochastic} linear bandit setting with  context distribution and unknown contexts, i.e., a bandit problem with performance constraints and unknown contexts. We assume that the context at time $t$, $c_t$ is unobservable rather only a distribution of the context denoted as $\mu_t$ is observed by the agent. At round $t$, the environment chooses a distribution $\mu_t \in \P(\C)$ over the context set and samples a context realization $c_t \sim \mu_t$.
The learner observes only $\mu_t$ and not  $c_t$ and chooses an action, say $x_t$.  In addition, there exists a baseline  policy (farmer's strategy) $\pi_b$    that at each round $t$, selects
action $b_t \in \X$ and incurs the expected reward $r(b_t, c_t) = \langle\theta^\star, \phi_{b_t, c_t}  \rangle$. We assume that the expected  rewards of the actions taken by the baseline policy, $r(b_t, c_t)$, are known.  This assumption is often reasonable as we typically have access to a large amount of data generated using the baseline policy (i.e., the farmer's strategy)  and hence can obtain a good estimate of the baseline reward function \cite{kazerouni2016conservative}.

Based on the baseline policy, a conservative linear bandit imposes performance constraints. The constraints are such that at round $t$, the difference between the performances of the baseline and the learner's policies should remain above a pre-defined fraction $\alpha \in (0, 1)$ of the baseline performance.  
Our aim is to learn an optimal mapping/policy $g: \C \rightarrow \X$ of contexts  to actions  such that the cumulative reward, $\sum_{t=1}^Ty_t$ is maximized while simultaneously satisfying the performance constraints. Formally, our aim is to minimize the cumulative regret 
 
 \begin{equation}
 \R_T = \sum_{t=1}^T\langle\theta^\star, \phi_{\xs_t, c_t} \rangle - \sum_{t=1}^T\langle\theta^\star, \phi_{x_t, c_t} \rangle.\label{eq:regret-sc}
 \end{equation}
 such that 
 \begin{equation}\label{eq:constraint}
\sum_{i=1}^t  r(b_t, c_t) -  \sum_{i=1}^t  r(x_t, c_t) \leqslant \alpha \sum_{i=1}^t  r(b_t, c_t), \mathrm{~for~all}~ t \in T.
 \end{equation}
Here, $\xs_t = \arg\max_{x \in \X} \mathbb{E}_{c \sim \mu_t}[r(x, c)]$ is the best action provided we know $\mu_t$, but not $c_t$, $T$ is the number of rounds, and $\alpha \in (0, 1)$ is the maximum decrease in the performance the decision maker is willing to accept. 
Eq.~\eqref{eq:constraint} is equivalent to  $\sum_{i=1}^t  r(x_t, c_t)  \geqslant (1-\alpha) \sum_{i=1}^t  r(b_t, c_t)$.

\section{Related Work}\label{sec:rel}
%The multi-armed bandit problem has been extensively studied and different solution approaches have been proposed using  stochastic formulation \cite{lai1985asymptotically, auer2002finite} and  Bayesian formulation \cite{agrawal2012analysis}. However, these approaches did not take into account the context or the side information available and  did not assume any specific structure on the reward function. Contextual bandits with linear reward function is studied in LINUCB \cite{}

Bandit algorithms are well studied in the literature, for a survey see \cite{bubeck2012regret} and \cite{lattimore2020bandit}.  Recently, contextual bandits have attracted increased attention. Related to our work is stochastic  contextual bandits, where the learner chooses actions after observing the contexts and the goal is to learn an optimal mapping from contexts to actions.  While stochastic contextual bandits have similarities to  Reinforcement Learning (RL) \cite{sutton2018reinforcement}, the key difference is that the sequence of contexts can be arbitrary and even chosen by an adversary unlike in an RL setting which has a specific transition structure.   Linear contextual bandits is a popular variant of the contextual bandits and it has been studied in \cite{abbasi2011improved, auer2002using, dani2008stochastic, li2010contextual, chu2011contextual, agrawal2013thompson, allesiardo2014neural} and strong  theoretical guarantees are established using different solution approaches. The most popular solution approach is the Upper Confidence Bound (UCB) algorithm \cite{li2010contextual, auer2002finite}. The UCB was later improved in \cite{abbasi2011improved, dani2008stochastic, li2021tight} with stronger guarantees.  Another solution approach is using Thompson sampling and algorithms with theoretical guarantees are provided in \cite{agrawal2013thompson}.  In the linear contextual bandit setting,  there are no constraints that need to be satisfied by the learner and the context in round $t$ is known and hence it is a special case of the bandit setting considered in this paper with  no constraints and the choice of the distribution $\mu_t$ as a Dirac delta distribution denoted as $\mu_t = \delta_{c_t}$ for all $t \in T$.  
 
Our work is more closely related to two settings of the contextual bandit problem,  the stochastic bandit framework and  the constrained contextual bandit framework.
A linear contextual bandit setting with uncertainty in the context is  studied in  \cite{kirschner2019stochastic, lamprier2018profile, yun2017contextual}.  While \cite{yun2017contextual} considered a setting with perturbed contexts, \cite{kirschner2019stochastic}  considered a setting  in which the context itself is not observable rather a distribution on the context is available and is  more closely related to this work. We note that, there are no safety constraints in \cite{kirschner2019stochastic}. There are two different settings where constraints have been applied to the stochastic MAB problem \cite{kazerouni2016conservative, wu2016conservative, badanidiyuru2014resourceful, amani2019linear, russo2014learning, daulton2019thompson}. The first line of work considers the MAB problem with global budget constraints where each arm is associated with a
random resource consumption and the objective is to maximize the total reward before the learner
exhausts all of its resources \cite{badanidiyuru2014resourceful, agrawal2014bandits}. Constrained linear bandit with linear budget constraints is studied in \cite{badanidiyuru2014resourceful} and a primal-dual algorithm is presented. A generalized version of the problem studied in \cite{badanidiyuru2014resourceful}, where the objective is concave and constraints are
convex is studied in \cite{agrawal2014bandits} and a UCB-based algorithm was proposed. We note that, the constraints in  \cite{badanidiyuru2014resourceful, agrawal2014bandits} are modeled as budget constraints unlike in this paper which consider a performance constraint. The second line of work considers safety/performance constraints for bandit problems by ensuring that the performance of the learning algorithm should remain above a pre-defined fraction of the performance of a baseline policy \cite{kazerouni2016conservative, wu2016conservative, amani2019linear}. Among these our work is closely related to \cite{kazerouni2016conservative}, but the key difference is that the contextx are unknown in our setting.
 In this paper we built on the works in \cite{kirschner2019stochastic, kazerouni2016conservative} to address the conservative and stochastic contextual bandit  problem in which the contexts are uncertain and the learner is subject to performance constraints imposed by some baseline policy. 
 %We modify the UCB algorithm for stochastic bandits with context distribution in \cite{kirschner2019stochastic} such that the action obtained as output of  the algorithm is guaranteed to satisfy the constraints. 

\section{Solution Approach:  Stochastic Conservative Contextual Bandit} \label{sec:sol} 
In this section, we present the algorithm for solving the stochastic conservative contextual bandit problem. Our solution approach is built on the works of  \cite{kirschner2019stochastic} and \cite{kazerouni2016conservative}. Given the distribution $\mu_t$, we construct the expected feature vector, $\Psi_t = \{\bar{ \psi}_{x, \mu_t}: x \in \X\}$ where $\{\bar{ \psi}_{x, \mu_t} := \mathbb{E}_{c \sim \mu_t}[\phi_{x,c}]\}$ (step:~\ref{step:psi}). We note that,  each feature $\bar{ \psi}_{x, \mu_t}$ corresponds to exactly one action $x \in \X$ and we use $\Psi_t $ as the feature context set at time $t$. The proposed algorithm is based on the {\em optimism in the face of uncertainty} principle, where the algorithm maintains a confidence set $\B_{t} \subset \R^d$ that contains the unknown parameter vector $\theta^\star$ with high probability \cite{abbasi2011improved}. The algorithm then chooses an optimistic estimate $\tilde{\theta}_t = \arg\max_{\hat{\theta} \in \B_t}~(\max_{x \in \X}~ \bar{\psi}_{x, \mu_t}^{\top}\hat{\theta})$ and  chooses an action $x'_t = \arg\max_{x \in \X}  \bar{\psi}_{x, \mu_t}^{\top}\tilde{\theta}_t$. Equivalently the algorithm chooses the pair
$(x'_t, \tilde{\theta}_t) \in \arg\max\limits_{(x,\hat{\theta}) \in \X \times \B_t}  \bar{ \psi}_{x, \mu_t}^{\top} \hat{\theta}$ which jointly maximizes the reward.
 
%An optimal action $x'_t$ is  then chosen  as 
 % \begin{equation}\label{eq:action}
% x'_t \in \arg\max_{x \in \X}~\max_{\hat{\theta} \in \B_t}~ \bar{\psi}_{x, \mu_t}^{\top}\hat{\theta},
% \end{equation}
%where  $x'_t$ has the best performance among all possible actions in $\X$, within the confidence set $\B_t$.

To ensure that the action chosen by the  algorithm guarantees satisfaction of the constraints,  the algorithm plays the action $x'_t$ only if  it satisfies the constraint for the worst choice of the parameter $\hat{\theta} \in \B_t$ \cite{kazerouni2016conservative}. We formally define this by introducing two sets $S_{t-1}^b$ and $S_{t-1}$. Let $S_{t-1}$ be the set of rounds $i$ before round $t$ at which the algorithm has played the optimistic action, i.e., $x_i=x'_i$. Then $S^b_{t-1}= \{1,2,\ldots, t-1\}-S_{t-1}$ is the set of rounds $j$ before round $t$ at which the algorithm has followed the baseline policy, i.e., $x_j = b_j$. To ensure that constraint in Eq.~\eqref{eq:constraint} is satisfied the algorithm plays optimal action $x_t=x'_t$ at round $t$  if it satisfies

{\scalefont{0.95}{
%\begin{equation}
\begin{equation*}
\min_{\hat{\theta} \in \B_t}\Big[\hspace*{-2 mm} \sum_{i \in S^b_{t-1}} \hspace*{-0mm}r(b_t, c_t)+(\hspace*{-1.5 mm} \sum_{i \in S_{t-1}}\bar{\psi}_{x_i, \mu_i})^{\top} \hat{\theta} + \bar{\psi}_{x'_t, \mu_t}^{\top} \hat{\theta}  \Big] \hspace*{-1 mm} \geqslant \hspace*{-1 mm} (1-\alpha) \sum_{i=1}^t r(b_i, c_i),
\end{equation*}
}
and plays the action chosen by the farmer, i.e., $x_t=b_t$ otherwise. 

\begin{algorithm}[h]
\caption{Pseudocode for conservative stochastic contextual bandit with context distribution}\label{alg:UCB}
\begin{algorithmic}
\State \textit {\bf Input:} $\alpha, \B=\mathbb{R}^d$
\end{algorithmic}
\begin{algorithmic}[1] %[1] enables line numbers
\State  \textit {\bf  Initialize:}  $S_0 = \emptyset, \ell_0 = 0 \in \bR^d$,  $\B_1 = \B$
\For{$t=1,2,\ldots, T$}
\State Nature chooses $\mu_t \in \P(\C) $
\State Learner observes $\mu_t$
\State Set $\Psi_t = \{\bar{ \psi}_{x, \mu_t}: x \in \X\}$ where $\{\bar{ \psi}_{x, \mu_t} := \mathbb{E}_{c \sim \mu_t}[\phi_{x,c}]\}$\label{step:psi}
\State Query  baseline  strategy $b_t \leftarrow \pi(\Psi_t)$\label{step:query}

\State Find $(x'_t, \tilde{\theta}_t) \in \arg\max\limits_{(x,\hat{\theta}) \in \X \times \B_t}  \bar{ \psi}_{x, \mu_t}^{\top} \hat{\theta}$
\State Compute $L_t = \min_{\hat{\theta} \in \B_t} \langle  \ell_{t-1}+\bar{ \psi}_{x'_t, \mu_t} , \hat{\theta}\rangle$
\If {$L_t+\sum_{i \in S_{t-1}^b} r(b_t, c_t)  \geqslant (1-\alpha) \sum_{i=1}^t r(b_t, c_t)$}
\State Play $x_t = x'_t$ and observe reward $y_t$ in Eq.~\eqref{eq:reward}
\State Set $\ell_t = \ell_{t-1}+ \bar{ \psi}_{x_t, \mu_t}$, $S_t =S_{t-1} \cup t$, $S^b_t =S^b_{t-1}$
\State Given $x_t, y_t$ construct  $\B_{t+1}$ using  Eq.~\eqref{eq:conf}
\Else
\State  Play $x_t = b_t$ and observe reward $y_t$ in Eq.~\eqref{eq:reward}
\State Set $\ell_t = \ell_{t-1}$, $S_t =S_{t-1}$, $S^b_t =S^b_{t-1} \cup t$, $\B_{t+1}=\B_t$
%\STATE Query  baseline (farmer's) strategy $b_t \leftarrow \pi_f(\Psi_t)$\label{step:query}

\EndIf
%\STATE $\X_{z} = \{x \in \X| \psi_{x, \mu_t}^{\top}\hat{\theta}_{t-1} +  \sqrt{\beta_t \psi_{x, \mu_t}^{\top}V_{t-1}^{-1}\psi_{x, \mu_t}} \geqslant (1-\alpha) \psi_{b_t, \mu_t}^{\top}\hat{\theta}_{t-1} +  \sqrt{\beta_t \psi_{b_t, \mu_t}^{\top}V_{t-1}^{-1}\psi_{b_t, \mu_t}}$\label{step:violation}

%\STATE Choose action $x_t =\arg\min\limits_{x \in \X_z} \psi_{x, \mu_t}^{\top}\hat{\theta}_{t-1} +  \sqrt{\beta_t \psi_{x, \mu_t}^{\top}V_{t-1}^{-1}\psi_{x, \mu_t}}$\label{step:UCBC}
%\STATE Environment provides $y_t = \phi_{x_t, c_t}^{\top}\theta+\epsilon$ where $c_t \sim \mu_t$\label{step:reward}
%\STATE Update $V_t = V_{t-1}+ \psi_{x_t, \mu_t} \psi_{x_t, \mu_t}^{\top}$ and $\hat{\theta}_t = V_t^{-1}+\sum_{\tau=1}^t \psi_{x_{\tau}, \mu_{\tau}}$\label{step:update}
\EndFor
\end{algorithmic}
\end{algorithm}

\noindent{\bf Construction of the Confidence Set $\B_t$:}
We denote the confidence set in round $t$ as $\B_t$. The proposed algorithm starts by the most general confidence set i.e., $\B_1 = \B = \mathbb{R}^d$, and updates the confidence set only when the optimistic action proposed by the learner is played. This is because that unless the learner's action is played, no additional information is gained about the unknown parameter $\theta$. Let $S_t = \{i_1, i_2, \ldots, i_{m_t}\}$ be the set of rounds up to and including $t$ during which the the algorithm played the optimistic action. Here $m_t = |S_t|$. For a fixed value $\lambda > 0$, the regularized least square estimate of $\hat{\theta}$ at round $t$ is given by
\begin{equation}
\bar{\theta}_t = \Big( \Phi_t\Phi_t^{\top}+\lambda I \Big)^{-1}\Phi_t Y_t,
\end{equation}
where $\Phi_t = [\bar{\psi}_{x_{i_1}, \mu_{i_1}}, \bar{\psi}_{x_{i_2}, \mu_{i_2}}, \ldots, \bar{\psi}_{x_{m_t}, \mu_{m_t}}]$ and $Y_t = [y_{i_1}, y_{i_2}, \ldots, y_{m_t}]^{\top}$. For a given confidence parameter $\delta \in (0,1 )$, we construct the confidence set for the next round $t+1$ as

\begin{equation}\label{eq:conf}
\B_{t+1} = \{\hat{\theta} \in \R^d:\norm{\hat{\theta}- \bar{\theta}_t}_{V_t} \leq \beta_{t+1}\},
\end{equation}
where $\beta_{t+1} = \sigma \sqrt{d\log (\dfrac{1+(m_t+1)D^2/\lambda}{\delta})} + \sqrt{\lambda} A$, $V_t = \lambda I + \Phi_t \Phi_t^{\top}$, and the weighted norm is defined as $\norm{u}_V = \sqrt{u^{\top}Vu}$ for any $u \in \R^d$ and positive definite $V \in \R^{d \times d}$.

\begin{proposition}\label{prop:confidence}
For any $\delta > 0$ and the confidence set $\B_t$ defined by Eq.~\eqref{eq:conf}, we have
$$\mathbb{P}[\theta^{\star} \in \B_t, \forall t \in \mathbb{N}] \geqslant 1-\delta.$$
\end{proposition}
At each round $t$, Algorithm~\ref{alg:UCB} ensures that Eq.~\eqref{eq:constraint} holds for all $\theta \in \B_t$. From Proposition~\ref{prop:confidence}, $ \mathbb{P}[\theta^{\star} \in \B_t] \geqslant 1-\delta$ for all $t \in \mathbb{N}$. Thus, Proposition~\ref{prop:confidence} ensures that at each round $t$, Algorithm~\ref{alg:UCB} satisfies the baseline criteria in Eq.~\eqref{eq:constraint} with probability at least $1-\delta$. 

\section{Regret Analysis}\label{sec:reg}
In this section, we prove the regret bound for Algorithm~\ref{alg:UCB}. 
%Recall that $x_t$ is the action of Algorithm~\ref{alg:UCB} at time $t$, $b_t$ is the baseline at time $t$, and  $x'_t$ is the learner's suggestion at time $t$. Also,  $\psi_{x_t, \mu_t}$ is the feature vector corresponding to action $x_t$ and context distribution $\mu_t$ and we define $\psi^{\star}_t = \arg\max_{\psi \in \Psi_t} \psi^{\top} \theta$.
 Let $\Delta_{b_t}^t = r(x^{\star}_t, c_t) - r(b_t, c_t)$ be the baseline gap at round $t$, i.e., the difference between the expected rewards of optimal action and baseline action at round $t$.

\begin{assume}\label{assume:delta}
There exists $0 \leq \Delta_{\ell} \leq \Delta_h$ and $0 < r_{\ell} < r_h$ such that, at each round $t$, 
$$\Delta_{\ell} \leq \Delta_{b_t}^t \leq  \Delta_h \mbox{~and~}  r_{\ell} \leq r(b^t, c_t) \leq  r_h.$$
\end{assume}
Since the rewards belong to $[0, 1]$ (Assumption~\ref{assume:bound}), we set $\Delta_h = r_ h =1$, and $\Delta_{\ell}=0$. The reward lower bound $r_l$ ensures that the baseline policy satisfies a minimum level of performance guarantee at each round of the algorithm. 
\begin{proposition}[\cite{kirschner2019stochastic}, Lemma~3]\label{prop:karuse}
The regret of the UCB algorithm for linear stochastic bandits  with expected feature set $\Psi_t$ is bounded in time $T$ with probability at least $1-\delta$, 
\begin{equation*}
 \R_T \leq  \R_T^{UCB}+4\sqrt{2T\log\dfrac{1}{\delta}}.
\end{equation*}
\end{proposition}

\begin{lemma}\label{lem:one}
The regret of Algorithm~\ref{alg:UCB} with expected feature set $\Psi_t$ is bounded in time $T$ with probability at least $1-\delta$,
$$\R_T \leq   \R_{S_T}^{UCB} + 4\sqrt{2m_T\log\dfrac{1}{\delta}}+ n_{T} \Delta_h,$$
where $ \R_{S_T}^{UCB}$ is the cumulative (pseudo)-regret of linear UCB algorithm at rounds $t\in S_T$, $m_T = |S_T|$ is the number of times Algorithm~\ref{alg:UCB} played the learner's action, and $n_T = |S_t^b|=T-|S_T|=T-m_T$ is the number of times Algorithm~\ref{alg:UCB} played the baseline action.
  \end{lemma}
\begin{proof}
From the definition of regret in Eq.~\eqref{eq:regret}
 \begin{eqnarray}
 \R_T &=&\hspace*{-2.5 mm}\sum_{t=1}^T r(\xs_t, c_t)- \sum_{t=1}^T r(x_t, c_t),\nonumber\\
 &=&\hspace*{-2.5 mm}\sum_{t \in S_T}\hspace*{-1 mm}(r(\xs_t, c_t)- r(x_t, c_t))+ \hspace*{-2 mm}\sum_{t\in S^b_T} (r(\xs_t, c_t)\hspace*{-1 mm}-\hspace*{-1 mm} r(x_t, c_t)),\nonumber\\
  &=&\hspace*{-2.5 mm}\sum_{t \in S_T}(r(\xs_t, c_t)- r(x_t, c_t))+  \sum_{t\in S^b_T} \Delta_{b_t}^t,\nonumber\\
   &\leq &\hspace*{-2.5 mm}\sum_{t \in S_T}(r(\xs_t, c_t)- r(x_t, c_t))+ n_{T} \Delta_h,\nonumber\\
  &\leq &\hspace*{-2.5 mm} \R_{S_T}^{UCB} + 4\sqrt{2m_T\log\dfrac{1}{\delta}}+ n_{T} \Delta_h.\label{eq:bound1}
 \end{eqnarray}
Inequality in Eq.~\eqref{eq:bound1} follows from Proposition~\ref{prop:karuse} since for $t \in S_T$, Algorithm~\ref{alg:UCB} plays the same actions as the UCB algorithm in \cite{kirschner2019stochastic}  and this completes the proof.
\end{proof}

The regret bound for linear UCB algorithm for the confidence set given in Eq.~\eqref{eq:conf} is  given in \cite{abbasi2011improved}. Let $\varepsilon$ be the event that $\theta^{\star} \in \B_t$ for all $t \in \mathbb{N}$. By Proposition~\ref{prop:confidence} the probability of $\varepsilon$ is at least $1-\delta$. The result below from \cite{abbasi2011improved} presents the bound for $\R_{S_T}^{UCB}$.

\begin{proposition}[\cite{kazerouni2016conservative}, Proposition~3]\label{prop:ucb-bound}
On event $\varepsilon$, for any $T \in \mathbb{N},$ we have
\begin{eqnarray}
\R_{S_T}^{UCB}  \hspace*{-2.5 mm}&\leq & \hspace*{-2.5 mm}4 \sqrt{m_T d \log\Big(\lambda+\dfrac{m_T D}{d}\Big)}\nonumber\\
&\times & \hspace*{-3 mm} \Big[A\sqrt{\lambda}+\sigma \sqrt{2\log(1/\delta)+d\log\Big(1+\dfrac{m_TD}{\lambda d}\Big)}\nonumber\\
&= & \hspace*{-3 mm} O\Big( d\log (\dfrac{D}{\lambda\delta}T)\sqrt{T} \Big).
\end{eqnarray}
\end{proposition}

We note that, to bound the regret of Algorithm~\ref{alg:UCB}, we only need to find upper bounds on $n_T$, the number of times Algorithm~\ref{alg:UCB} deviates from the UCB algorithm for linear stochastic bandits  and plays the baseline,  and $m_T$, the number of times Algorithm~\ref{alg:UCB} plays the action suggested by the UCB algorithm for linear stochastic bandits. Since $m_T = T-n_T$, it also suffices to find an upper and lower bounds for $n_T$. An upper bound for $n_T$ is given in \cite{kazerouni2016conservative} which is presented in the proposition below.

\begin{proposition}[\cite{kazerouni2016conservative}, Theorem~5]\label{prop:upper}
Assume that $\lambda \geqslant \max\{1,D^2\}$. On event $\varepsilon$, for any horizon $T \in \mathbb{N}$, we have
$$n_T \leqslant 1+114d^2 \dfrac{(A\sqrt{\lambda}+\sigma)^2}{\alpha r_{\ell}(\Delta_{\ell}+\alpha r_{\ell})}\Big[ \log \Big ( \dfrac{62d(A\sqrt{\lambda}+\sigma)}{\sqrt{\delta}(\Delta_{\ell}+\alpha r_{\ell})}\Big) \Big]^2.$$
\end{proposition} 

Thus the only thing remaining to prove is a lower bound on $n_T$. To prove a lower bounds on $n_T$, we use  Proposition~\ref{prop:min} from \cite{kazerouni2016conservative} and Lemma~\ref{lem:g}.

\begin{proposition}[\cite{kazerouni2016conservative}, Lemma~4]\label{prop:min}
For given $k \in \mathbb{N}$, $\lambda > 0$, and any sequence $Y_1, Y_2, \ldots, Y_k$ in $\mathbb{R}^d$ such that for all $i: \norm{Y_i}_2 \leqslant D$, let $V_0 = \lambda I$ and $V_i = \lambda I +\sum_{j=1} Y_j Y_i^{\top}$ for $1\leqslant i \leq k$. Then, we have 
\begin{equation}
\sum_{i=1}^k \min\Big(1, \norm{Y_i}^2_{V^{-1}_{i-1}} \Big) \leqslant 2d \log \Big( 1+\dfrac{kD^2}{\lambda d} \Big).
\end{equation}
\end{proposition} 

\begin{lemma}\label{lem:g}
For any $m \geq 2$ and $c_1, c_2, c_3 > 0$,  $-c_3 m-c_1\sqrt{m}\log(c_2 m) \geq \dfrac{16c_1^2}{25c_3}\Big[\log(\dfrac{2c_1\sqrt{c_2}e}{c_3})\Big]^2.$
\end{lemma}
\begin{proof}
Let $g(m) = -c_3 m-c_1\sqrt{m}\log(c_2 m)$. Then, $g'(m)=-c_3-\dfrac{c_1(2+\log(c_2 m))}{2\sqrt{m}}$ and $g"(m) = \dfrac{c_1 \log(c_2 m)}{4m\sqrt{m}}$. Since $c_2>1$, $g$ is a convex function over its domain $[2, \infty)$, and thus a global optimum $m^{\star}$ exists  for $g$. By the first order condition, we get $g'(m^{\star})=0$. This gives
\begin{equation}\label{eq:app1}
2+\log(c_2m^{\star}) = \dfrac{-2c_3}{c_1}\sqrt{m^{\star}}.
\end{equation} 

Thus $g^{\star} = g(m^{\star})=c_3 m^{\star}+2c_1\sqrt{m^{\star}}$. Using  change of variables $z=\dfrac{c_3}{2c_1}\sqrt{m^{\star}}$, we get
\begin{equation}\label{eq:app2}
g^{\star}=\dfrac{4c_1^2}{c_3}(z^2+z).
\end{equation} 
Eq.~\eqref{eq:app1} becomes 
$$2+\log(\dfrac{4c_2 c_1^2}{c_3^2})+2\log(z) = -4z.$$
After taking exponential on both sides, 
$$\dfrac{e^{-4z}}{z^2} = \dfrac{4c_1^2c_2e^2}{c_3^2}.$$
Using $e^{z}> z^2$,
$$ \dfrac{4c_1^2c_2e^2}{c_3^2} =\dfrac{e^{-4z}}{z^2} > \dfrac{e^{-4z}}{e^z} =e^{-5z}.$$
Thus $$z\geq\dfrac{-1}{5} \log(\dfrac{4c_1^2c_2e^2}{c_3^2}).$$ Substituting in Eq.~\eqref{eq:app2}, we get
$$g^{\star}\hspace*{-1mm} \geq\hspace*{-1mm} \dfrac{4c_1^2}{c_3}z^2 \geq \hspace*{-1mm}\dfrac{4c_1^2}{25c_3}\Big[\hspace*{-1mm}\log(\dfrac{4c_1^2c_2e^2}{c_3^2})\Big]^2\hspace*{-2mm}=\hspace*{-1mm}\dfrac{16c_1^2}{25c_3}\Big[\hspace*{-1mm}\log(\dfrac{2c_1\sqrt{c_2}e}{c_3})\Big]^2.$$
\end{proof}

% such that 
% \begin{equation}
%\sum_{i=1}^t  f(x_t, c_t)   \geqslant (1-\alpha) \sum_{i=1}^t  z(b_t, c_t) \label{eq:constraint}
% \end{equation}

\begin{theorem}\label{th:lower}
Assume that $\lambda \geqslant D^2$. On event $\varepsilon$, for any horizon $T \in \mathbb{N}$, the following holds
$$ n_T \geqslant \dfrac{d^2 (A\sqrt{\lambda}+\sigma)^2}{\alpha r_h(\Delta_h +\alpha r_h)}\Big[ \log \Big(\dfrac{10d(A\sqrt{\lambda}+\sigma)}{\sqrt{\delta}(\Delta_h +\alpha r_h)} \Big) \Big]^2.$$
\end{theorem}
\begin{proof}
Let $\tau$ be the last round in which Algorithm~\ref{alg:UCB} plays the learner's action, $\tau = \max\{1\leqslant t \leq T|x_t=x'_t \}.$
\begin{equation*}
\min_{\theta \in \B_{\tau}} \langle \theta, \bp^{\tau}_{x'_{\tau}} +\sum_{t \in S_{\tau-1}}\bp^{t}_{x_{t}}  \rangle + \sum_{t \in S^b_{\tau-1}} r(b_t, c_t) \geqslant (1-\alpha) \sum_{t=1}^{\tau} r(b_t, c_t),
\end{equation*}
\begin{eqnarray}
\alpha  \sum_{t=1}^{\tau} r(b_t, c_t) \hspace*{-2.5 mm}&\geqslant &\hspace*{-4.5 mm} \sum_{t \in S_{\tau-1}} r(b_t, c_t) + r (b_{\tau}, c_\tau)- \min_{\theta \in \B_{\tau}} \langle \theta, \bp^{\tau}_{x'_{\tau}} +\sum_{t \in S_{\tau-1}}\bp^{t}_{x_{t}}  \rangle,\nonumber \\
\hspace*{-2.5 mm}&\geqslant &\hspace*{-4.0 mm}  \sum_{t \in S_{\tau-1}} (r(b_t, c_t) -\langle \theta^{\star}, \bp_{x_t}^t \rangle) + (r(b_{\tau}, c_\tau) -\langle \theta^{\star}, \bp_{x'_{\tau}}^{\tau} \rangle )\nonumber \\
\hspace*{-2.5 mm}&+&\hspace*{-3.5 mm}   \langle \theta^{\star}, \bp^{\tau}_{{x'_{\tau}}}\hspace*{-1.5 mm} +\hspace*{-2.5 mm} \sum_{t \in S_{\tau-1}}\hspace*{-2.5 mm}\bp_{x_t}^t \rangle -  \min_{\theta \in \B_{\tau}} \langle \theta, \bp^{\tau}_{x'_{\tau}}\hspace*{-1 mm}+\hspace*{-2 mm}\sum_{t \in S_{\tau-1}}\hspace*{-2.5 mm}\bp^{t}_{x_{t}}  \rangle\nonumber \\
%\hspace*{-2.5 mm}&\geqslant &\hspace*{-4.0 mm}\sum_{t \in S_{\tau-1}} \hspace*{-2.5 mm}(-\Delta_{b_t}^t)\hspace*{-1 mm} - \hspace*{-1 mm}\Delta^{\tau}_{b_{\tau}}\hspace*{-2.0 mm} +\hspace*{-0.7 mm} \max_{\theta \in \B_{\tau}}\langle\theta^{\star}\hspace*{-1.5 mm}-\hspace*{-1.0 mm}\theta, \bp^{\tau}_{x'_{\tau}}\hspace*{-1 mm}+\hspace*{-3 mm}\sum_{t \in S_{\tau-1}}\hspace*{-3.0 mm}\bp^{t}_{x_{t}}  \rangle\nonumber\\
\hspace*{-2.5 mm}&\geqslant &\hspace*{-4.0 mm}\sum_{t \in S_{\tau-1}} \hspace*{-2.5 mm}(-\Delta_{b_t}^t)\hspace*{-1 mm} - \hspace*{-1 mm}\Delta^{\tau}_{b_{\tau}}\hspace*{-2.0 mm} -\hspace*{-0.7 mm} \min_{\theta \in \B_{\tau}}\langle\theta, \bp^{\tau}_{x'_{\tau}}\hspace*{-1 mm}+\hspace*{-3 mm}\sum_{t \in S_{\tau-1}}\hspace*{-3.0 mm}\bp^{t}_{x_{t}}  \rangle\nonumber\\
\hspace*{-2.5 mm}&\geqslant &\hspace*{-4.0 mm}\sum_{t \in S_{\tau-1}} \hspace*{-2.5 mm}(-\Delta_h)\hspace*{-1 mm} - \Delta_h-\hspace*{-0.7 mm} \min_{\theta \in \B_{\tau}}\langle\theta, \bp^{\tau}_{x'_{\tau}}\hspace*{-1 mm}+\hspace*{-3 mm}\sum_{t \in S_{\tau-1}}\hspace*{-3.0 mm}\bp^{t}_{x_{t}}  \rangle\nonumber\\
\hspace*{-2.5 mm}&=&\hspace*{-2.5 mm} -(m_{\tau-1}\hspace*{-0.5 mm}+\hspace*{-0.5 mm}1)\Delta_h  \hspace*{-0.5 mm}- \hspace*{-1 mm}\min_{\theta \in \B_{\tau}}\langle \theta, \bp^{\tau}_{x'_{\tau}}\hspace*{-1 mm}+\hspace*{-3 mm}\sum_{t \in S_{\tau-1}}\hspace*{-3 mm}\bp^{t}_{x_{t}}  \rangle\label{eq:min0}\\
\hspace*{-2.5 mm}&\geqslant &\hspace*{-2.5 mm}-(m_{\tau-1}\hspace*{-1 mm}+\hspace*{-1 mm}1)\Delta_h \hspace*{-1 mm} -\norm{\theta}_{V_{\tau}} \hspace*{-1 mm}\norm{\bp^{\tau}_{x'_{\tau}}\hspace*{-0.5 mm}+\hspace*{-1.5 mm}\sum_{t \in S_{\tau-1}}\hspace*{-2.5 mm}\bp^{t}_{x_{t}}}_{V^{-1}_{\tau}}\nonumber\\
\hspace*{-2.5 mm}&\geqslant &  \hspace*{-2.5 mm}-(m_{\tau-1}\hspace*{-1 mm}+\hspace*{-1 mm}1)\Delta_h-\beta_{\tau} \norm{\bp^{\tau}_{x'_{\tau}}+\sum_{t \in S_{\tau-1}}\bp^{t}_{x_{t}}}_{V^{-1}_{\tau}}\nonumber\\
\hspace*{-2.5 mm}&\geqslant &\hspace*{-3 mm}-(m_{\tau-1}\hspace*{-1 mm}+\hspace*{-1 mm}1)\Delta_h \hspace*{-1 mm} -\hspace*{-1 mm}\beta_{\tau}(\norm{\bp^{\tau}_{x'_{\tau}}}_{V^{-1}_{\tau}} \hspace*{-2 mm}+\hspace*{-2.5 mm}\sum_{t \in S_{\tau-1}} \hspace*{-2 mm}\norm{\bp^{t}_{x_{t}}}_{V^{-1}_{t}})\label{eq:min1}
\end{eqnarray}
From Eq.~\eqref{eq:min0}, we get 
\begin{equation}\label{eq:min2}
\alpha  \sum_{t=1}^{\tau} r(b_t, c_t)  \geqslant  -(m_{\tau-1}+1)\Delta_h  - (m_{\tau -1}+1)
\end{equation}
We note that, $\beta_{\tau}$ is non decreasing and is greater than $1$. Hence from Eqs.~\eqref{eq:min1} and~\eqref{eq:min2}, 
\begin{eqnarray}
\alpha  \sum_{t=1}^{\tau} r(b_t, c_t) \geqslant -(m_{\tau-1}+1)\Delta_h-\beta_{\tau} [\min( \norm{\bp^{\tau}_{x'_{\tau}}}_{V^{-1}_{\tau}}, 1)\nonumber\\
\hspace*{-3 mm}+ \hspace*{-3 mm}\sum_{t \in S_{\tau-1}} \hspace*{-2 mm} \min(\norm{\bp^{t}_{x_{t}}}_{V^{-1}_{t}}, 1)]\label{eq:min3}
\end{eqnarray}
In order to simplify the equation, we introduce $\Gamma$ as
$$\Gamma := \Big[\min( \norm{\bp^{\tau}_{x'_{\tau}}}^2_{V^{-1}_{\tau}}, 1)+\sum_{t \in S_{\tau-1}} \min(\norm{\bp^{t}_{x_{t}}}^2_{V^{-1}_{t}}, 1)]  \Big].$$ By Cauchy-Schwarz inequality and using Proposition~\ref{prop:min}, and $\Gamma$, Eq.~\eqref{eq:min3} can be written as
\begin{eqnarray*}
%-\alpha  \sum_{t=1}^{\tau} r_{b}^t \hspace*{-3 mm}&\leq &\hspace*{-3 mm} (m_{\tau-1}\hspace*{-1mm} +\hspace*{-1 mm} 1)\Delta_h  \hspace*{-1 mm} -\hspace*{-1 mm}  2\beta_{\tau} [\min( \norm{\bp^{\tau}_{x'_{\tau}}}_{V^{-1}_{\tau}}, 1)\nonumber\\
%\hspace*{-2.5 mm} &+&  \hspace*{-2.5 mm}\sum_{t \in S_{\tau-1}} \min(\norm{\bp^{t}_{x_{t}}}_{V^{-1}_{t}}, 1)]\nonumber\\
\alpha  \sum_{t=1}^{\tau} r(b_t, c_t) \hspace*{-2.5 mm}&\geqslant &\hspace*{-2.5 mm} -(m_{\tau-1}+1)\Delta_h - \beta_{\tau} \sqrt{(m_{\tau-1}+1) \Gamma}
\end{eqnarray*}

\begin{eqnarray}
\hspace*{-2.5 mm}&\geqslant&\hspace*{-2.5 mm}-(m_{\tau-1}\hspace*{-1 mm}+\hspace*{-1 mm}1)\Delta_h\hspace*{-1 mm} - \hspace*{-1 mm}\beta_{\tau} \sqrt{\hspace*{-0.8 mm}2(m_{\tau-1}\hspace*{-1 mm}+\hspace*{-1 mm}1) d\log(\hspace*{-0.5 mm}1\hspace*{-1 mm}+\hspace*{-1 mm}\dfrac{(m_{\tau-1}\hspace*{-1 mm}+\hspace*{-1 mm}1)D^2}{\lambda d}\hspace*{-1 mm})}\nonumber\\
\hspace*{-2.5 mm}&=&\hspace*{-2.5 mm}-(m_{\tau-1}\hspace*{-1 mm}+\hspace*{-1 mm}1)\Delta_h\hspace*{-1 mm} - \hspace*{-1 mm}\sqrt{2(m_{\tau-1}\hspace*{-1 mm}+\hspace*{-1 mm}1) d\log(1\hspace*{-1 mm}+\hspace*{-1 mm}\dfrac{(m_{\tau-1}\hspace*{-1 mm}+\hspace*{-1 mm}1)D^2}{\lambda d})}\nonumber\\
\hspace*{-2.5 mm}&\times&\hspace*{-2.5 mm}  \Big(\sqrt{\lambda}A + \sigma \sqrt{d \log(\dfrac{1+(m_{\tau-1}+1)D^2/\lambda}{\delta})}\Big)\nonumber\\
\hspace*{-2.5 mm}&\geqslant &\hspace*{-2.5 mm}-(m_{\tau-1}\hspace*{-1 mm}+\hspace*{-1 mm}1)\Delta_h\hspace*{-1 mm} - \hspace*{-1 mm}\sqrt{2(m_{\tau-1}\hspace*{-1 mm}+\hspace*{-1 mm}1) d\log(1\hspace*{-1 mm}+\hspace*{-1 mm}\dfrac{(m_{\tau-1}\hspace*{-1 mm}+\hspace*{-1 mm}1)}{d})}\nonumber\\
\hspace*{-2.5 mm}&\times&\hspace*{-2.5 mm}  \Big(\sqrt{\lambda}A + \sigma \sqrt{d \log(\dfrac{1+(m_{\tau-1}+1)}{\delta})}\Big)\label{eq:min:x}\\
\hspace*{-2.5 mm}&\geqslant &\hspace*{-2.5 mm}-(m_{\tau-1}\hspace*{-1 mm}+\hspace*{-1 mm}1)\Delta_h \hspace*{-1 mm}-\hspace*{-1 mm} \Big(\sqrt{2}d\sqrt{m_{\tau-1}+1}(A\sqrt{\lambda}\hspace*{-1 mm}+\hspace*{-1 mm}\sigma) \Big)\nonumber\\
\hspace*{-2.5 mm}&\times &\hspace*{-2.5 mm}\log(\dfrac{2(m_{\tau-1}\hspace*{-1 mm}+\hspace*{-1 mm}1)}{\delta})\label{eq:minxx}
\end{eqnarray}
The inequality in Eq.~\eqref{eq:min:x} follows from $D^2 \geqslant \lambda$ and the inequality in Eq.~\eqref{eq:minxx} holds since
\begin{eqnarray*}
\Big(\sqrt{\lambda}A + \sigma \sqrt{d \log(\dfrac{1+(m_{\tau-1}+1)}{\delta})}\Big)\nonumber\\
 \leq \Big(A\sqrt{\lambda}+\sigma\Big) \sqrt{d \log(\dfrac{2(m_{\tau-1}+1)}{\delta})}
\end{eqnarray*}
Eq.~\eqref{eq:minxx} can be rewritten as
\begin{eqnarray}
\alpha  \sum_{t=1}^{\tau} r_h \hspace*{-3.5 mm} &\geqslant &\hspace*{-3.5 mm}  -(m_{\tau-1}\hspace*{-1 mm}+\hspace*{-1 mm}1)\Delta_h \hspace*{-1 mm}-\hspace*{-1 mm} \Big(\sqrt{2}d\sqrt{m_{\tau-1}\hspace*{-1 mm}+\hspace*{-1 mm}1}(A\sqrt{\lambda}+\sigma)\Big)\nonumber\\
\hspace*{-2.5 mm}&\times &\hspace*{-2.5 mm}\log(\dfrac{2(m_{\tau-1}\hspace*{-1 mm}+\hspace*{-1 mm}1)}{\delta})\nonumber\\
\alpha n_{\tau-1}r_h \hspace*{-2.5 mm}&\geqslant &\hspace*{-2.5 mm} -(m_{\tau-1}\hspace*{-1 mm}+\hspace*{-1 mm}1)(\Delta_h +\alpha r_h) \hspace*{-1 mm}-\hspace*{-1 mm} \Big(\sqrt{2}d\sqrt{m_{\tau-1}\hspace*{-1 mm}+\hspace*{-1 mm}1}\Big)\nonumber\\
\hspace*{-2.5 mm}&\times &\hspace*{-2.5 mm}(A\sqrt{\lambda}+\sigma) \log(\dfrac{2(m_{\tau-1}\hspace*{-1 mm}+\hspace*{-1 mm}1)}{\delta})\label{eq:min4}
\end{eqnarray}
Eq.~\eqref{eq:min4} follows after substituting  $\alpha  \sum_{t=1}^{\tau} r_h = \alpha(m_{\tau-1}+n_{\tau-1}+1)r_h$. 
To prove a lower bound on the LHS of Eq.~\eqref{eq:min4}, we first present a lower bound for the RHS of Eq.~\eqref{eq:min4}. Let $m=(m_{\tau-1}+1)$, $c_1 = \sqrt{2}d(A\sqrt{\lambda}+\sigma), c_2= \dfrac{2}{\delta}, c_3=(\Delta_h + \alpha r_h)$.  By Lemma~\ref{lem:g},
\begin{eqnarray*}
\alpha n_{\tau-1}r_h \hspace*{-2.5 mm}&\geqslant &\hspace*{-2.5 mm} \dfrac{d^2 (A\sqrt{\lambda}+\sigma)^2}{(\Delta_h +\alpha r_h)}\Big[ \log \Big(\dfrac{10d(A\sqrt{\lambda}+\sigma)}{\sqrt{\delta}(\Delta_h +\alpha r_h)} \Big) \Big]^2\nonumber
%n_T\hspace*{-2.5 mm}&\geqslant &\hspace*{-2.5 mm} \dfrac{d^2 (B\sqrt{\lambda}+\sigma)^2}{(\Delta_h +\alpha r_h)}\Big[ \log \Big(\dfrac{10d(B\sqrt{\lambda})D}{\sqrt{\lambda \delta}(\Delta_h +\alpha r_h)} \Big) \Big]
\end{eqnarray*}
The result follows as $n_T \geqslant n_{\tau}= n_{\tau-1}$.
\end{proof}

\begin{figure*}[h]
\begin{subfigure}[b]{0.3\textwidth}
\centering
\includegraphics[width = 1.1\textwidth, height =2in]{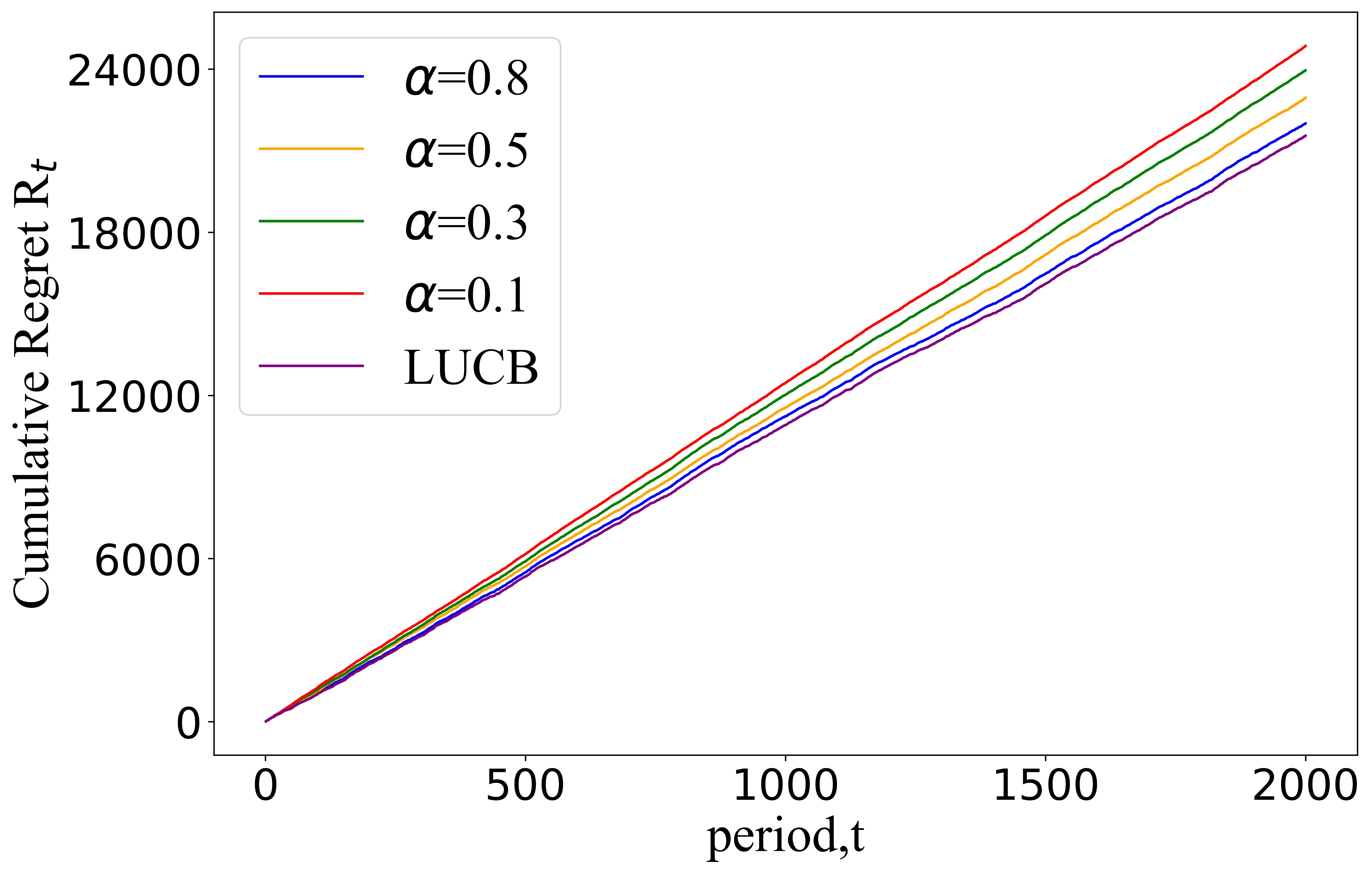}
\caption{}\label{fig:regret1}
\end{subfigure}\hfill
\begin{subfigure}[b]{0.3\textwidth}
\centering
\includegraphics[width = 1.1\textwidth, height =2in]{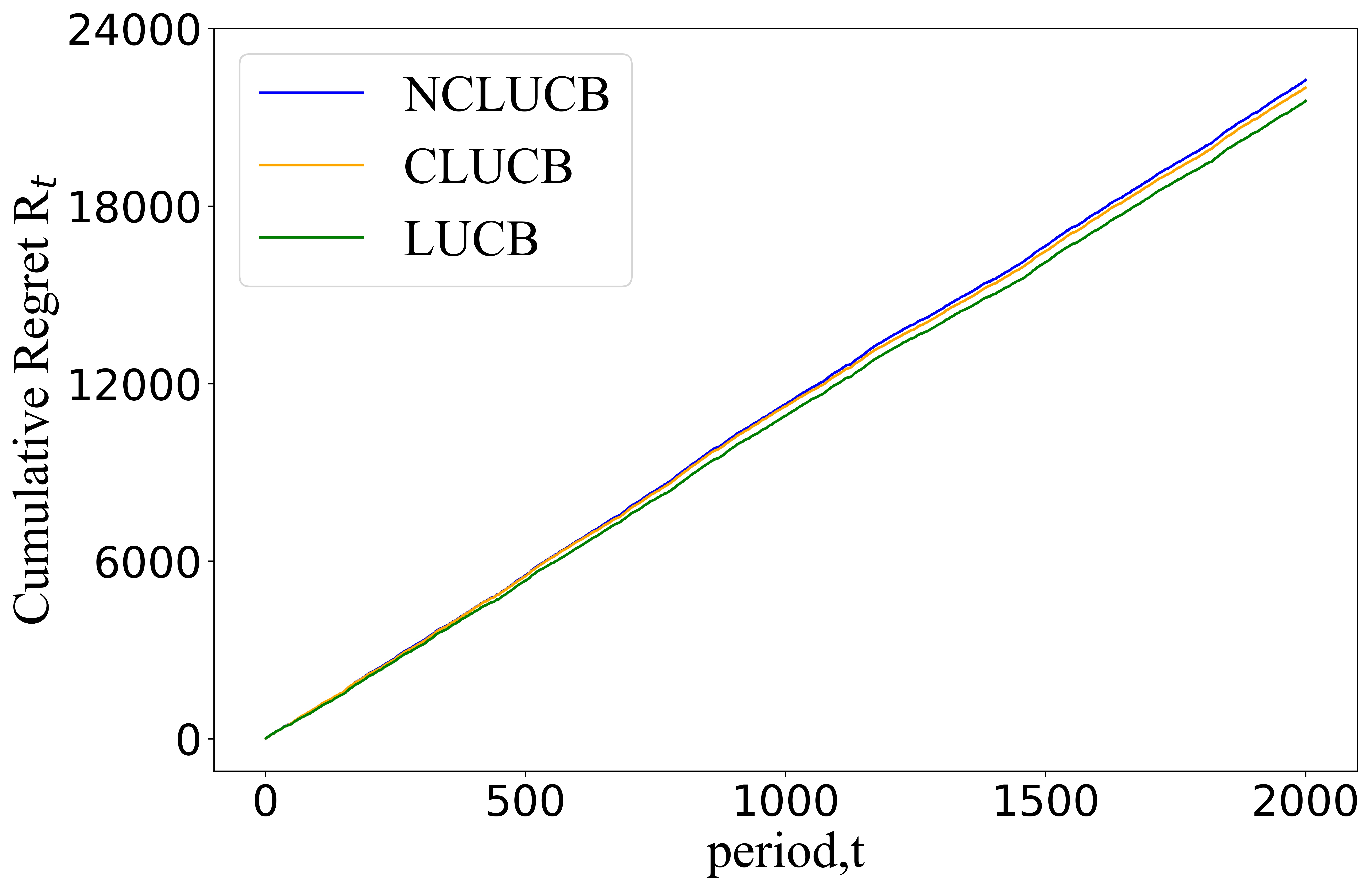}
\caption{}\label{fig:regret2}
\end{subfigure}\hfill
\begin{subfigure}[b]{0.3\textwidth}
\centering
\includegraphics[width = 1.1\textwidth, height =2in]{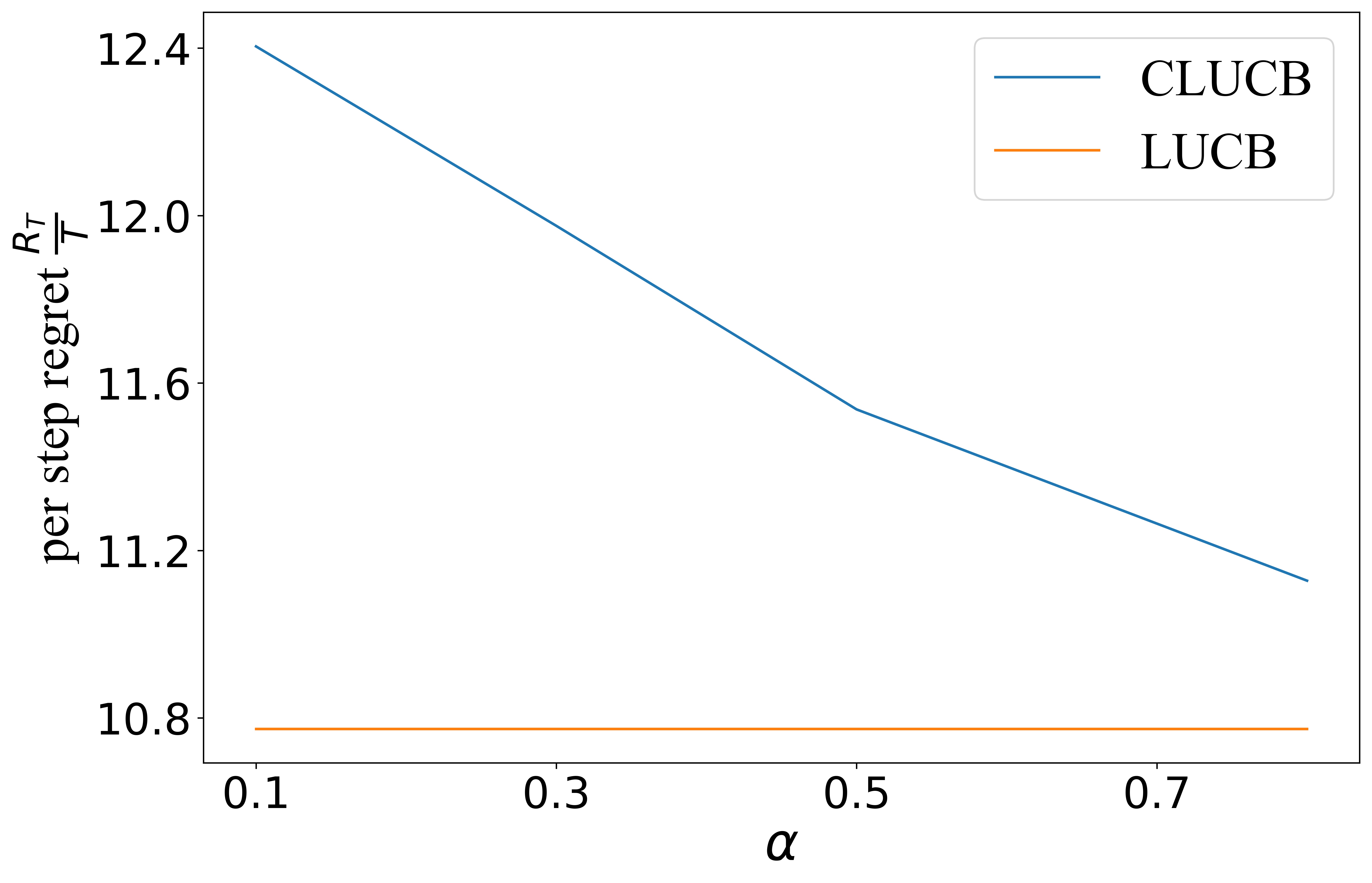}
\caption{}\label{fig:regret3}
\end{subfigure}
\caption{\small Plots for synthetic data. (a)~Cumulative regret of the standard linear UCB algorithm (LUCB \cite{abbasi2011improved}) and conservative stochastic bandit algorithm with context distribution (Algorithm~\ref{alg:UCB}) with $\alpha=0.1, 0.3, 0.5, 0.8$, (b)~Cumulative regret  for three settings: (i)~when the learner observes the context and there are no safety constraints (LUCB \cite{abbasi2011improved}), (ii)~when the learner observes the context and there are safety constraints (conservative linear UCB, CLUCB \cite{kazerouni2016conservative}), and (iii)~when the learner observes only the context distribution and there are safety constraints (Algorithm~\ref{alg:UCB}) (c)~Comparison of per step regret $\R_T/T$ at $T=2000$ for different values of $\alpha$ for (i), (ii), and (iii).}\label{fig:Regret-synthetic}
\end{figure*}
\begin{figure*}[h]
\begin{subfigure}[b]{0.3\textwidth}
\centering
\includegraphics[width = 1.1\textwidth, height =2in]{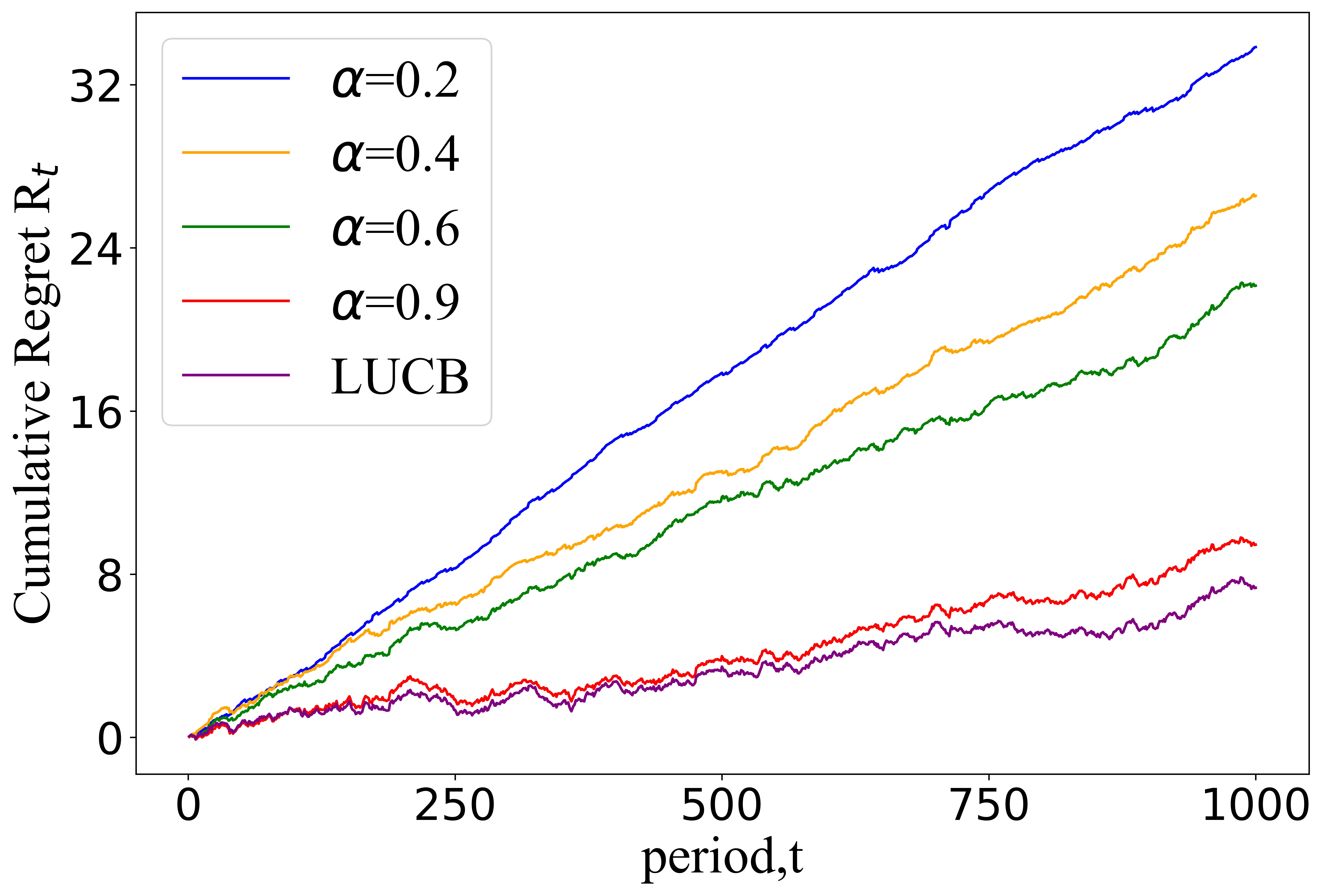}
\caption{}\label{fig:regret4}
\end{subfigure}\hfill
\begin{subfigure}[b]{0.3\textwidth}
\centering
\includegraphics[width = 1.1\textwidth, height =2in]{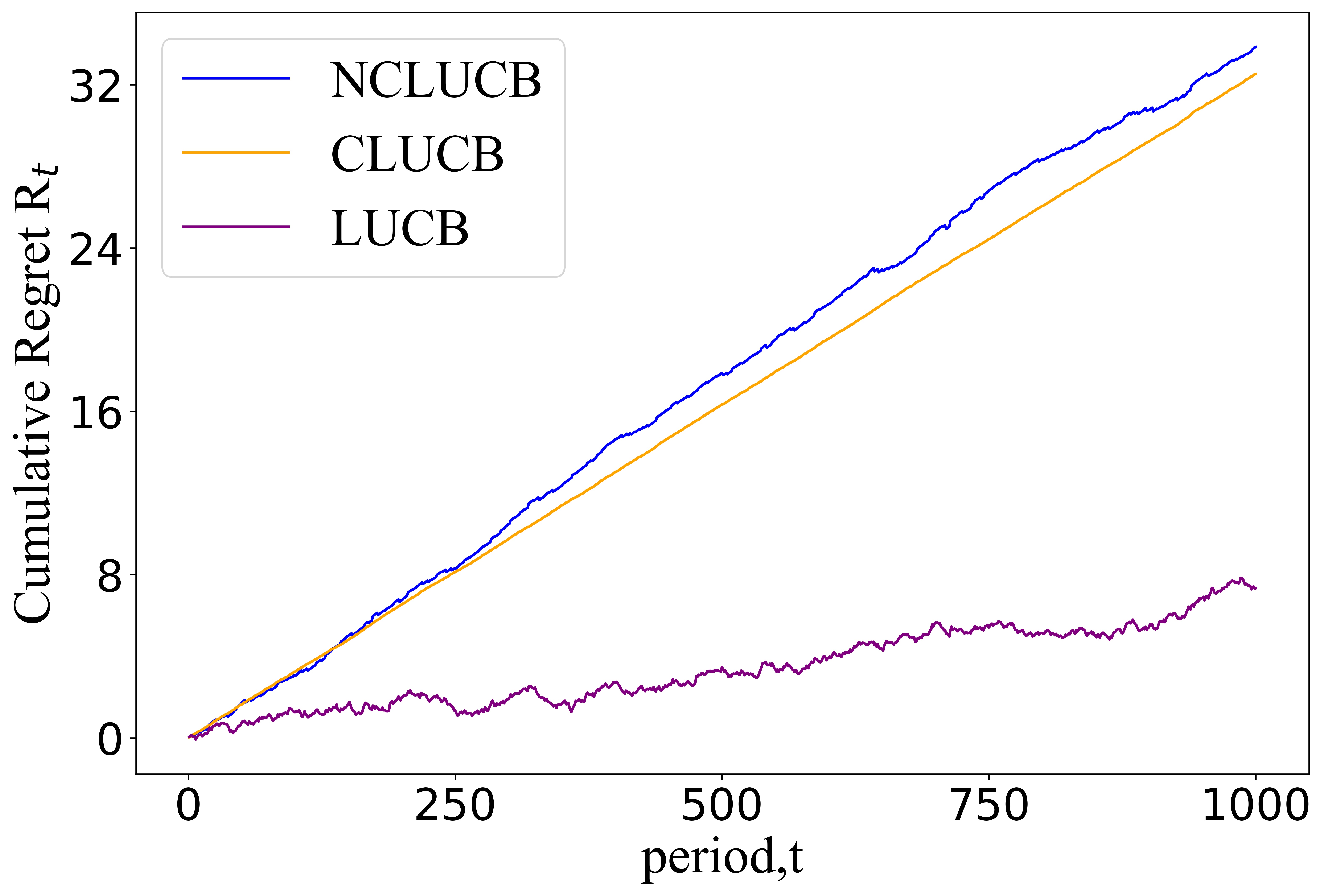}
\caption{}\label{fig:regret5}
\end{subfigure}\hfill
\begin{subfigure}[b]{0.3\textwidth}
\centering
\includegraphics[width = 1.1\textwidth, height =2in]{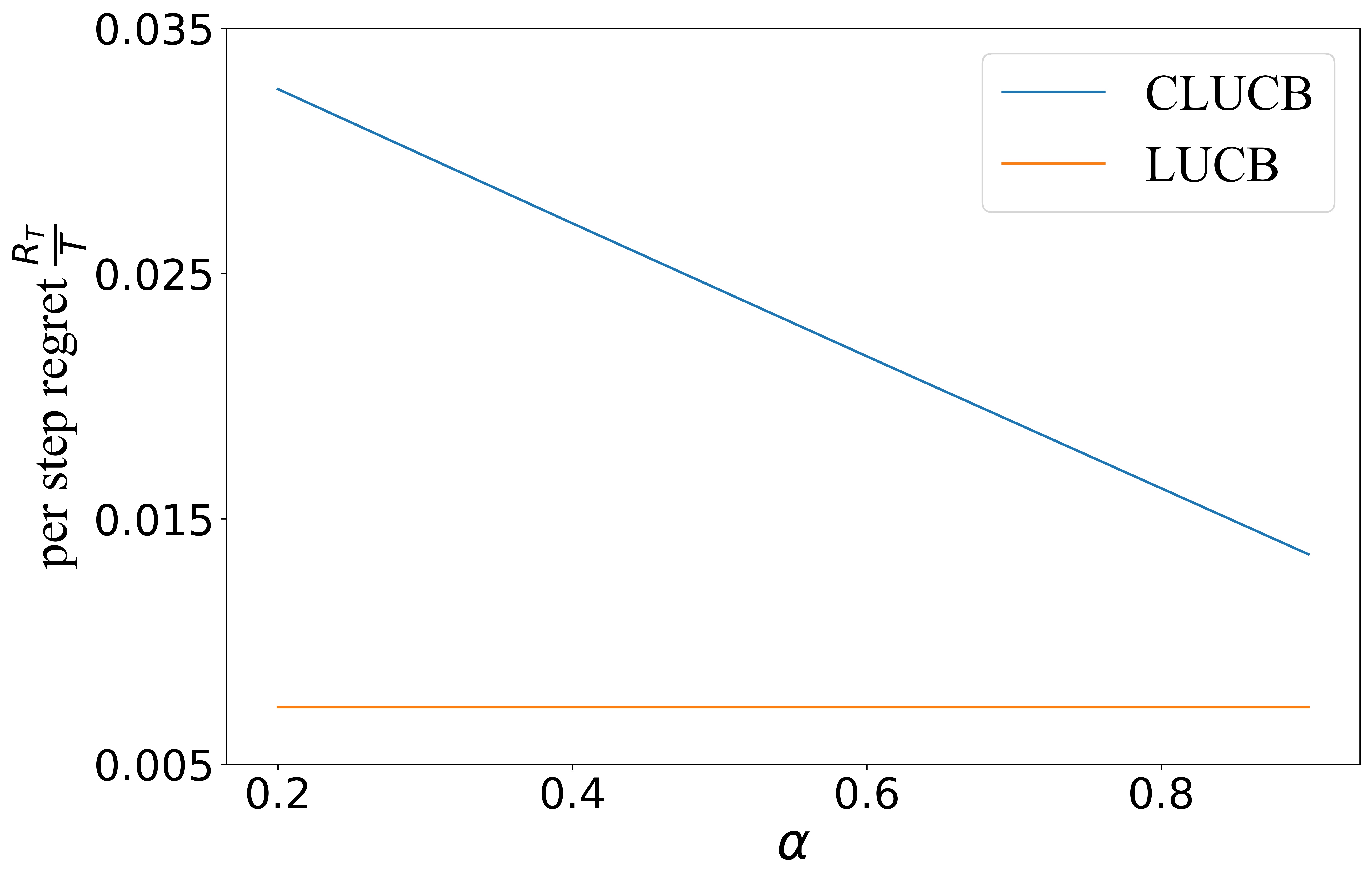}
\caption{}\label{fig:regret6}
\end{subfigure}
\caption{\small  Plots for maize yield data. (a)~Cumulative regret of the standard linear UCB algorithm (LUCB \cite{abbasi2011improved}) and conservative stochastic bandit algorithm with context distribution (Algorithm~\ref{alg:UCB}) with $\alpha=0.2, 0.4, 0.6, 0.9$, (b)~Cumulative regret  for three settings: (i)~when the learner observes the context and there are no safety constraints (LUCB \cite{abbasi2011improved}), (ii)~when the learner observes the context and there are safety constraints (conservative linear UCB, CLUCB \cite{kazerouni2016conservative}), and (iii)~when the learner observes only the context distribution and there are safety constraints (Algorithm~\ref{alg:UCB}) (c)~Comparison of per step regret $\R_T/T$ at $T=1000$ for different values of $\alpha$ for (i), (ii), and (iii).}\label{fig:Regret-Ag}
\end{figure*}

We now present the regret bound on Algorithm~\ref{alg:UCB} in Theorem~\ref{th:regret}. Proof of Theorem~\ref{th:regret} follows from Lemma~\ref{lem:one}, Proposition~\ref{prop:upper}, Theorem~\ref{th:lower}, and Proposition~\ref{prop:ucb-bound}. 

\begin{theorem}\label{th:regret}
With probability at least $1-\delta$, Algorithm~\ref{alg:UCB} satisfies the performace constraint in Eq.~\eqref{eq:constraint} for all $t \in \mathbb{N}$, and satisfies the following regret bound
$$\R_T \leq  O\Big( d\log (\dfrac{D}{\lambda\delta}T)\sqrt{T} +\dfrac{K_{h}\Delta_h}{\alpha r_{\ell}} + \dfrac{K_{\ell}\sqrt{\log(1/\delta)}}{\sqrt{\alpha r_h(\Delta_h +\alpha r_h)}}  \Big),$$
where $K_h$ and $K_{\ell}$ are constants that depend only on the parameters of the problem as $K_h =1+114d^2 \dfrac{(A\sqrt{\lambda}+\sigma)^2}{\Delta_{\ell}+\alpha r_{\ell}}\Big[ \log \Big ( \dfrac{62d(A\sqrt{\lambda}+\sigma)}{\sqrt{\delta}(\Delta_{\ell}+\alpha r_{\ell})}\Big) \Big]^2$ and $K_{\ell}= d (A\sqrt{\lambda}+\sigma)\Big[ \log \Big(\dfrac{10d(A\sqrt{\lambda}+\sigma)}{\sqrt{\delta}(\Delta_h +\alpha r_h)} \Big) \Big].$
\end{theorem}
\begin{proof}
The proof follows from Lemma~\ref{lem:one}, proposition~\ref{prop:ucb-bound}, proposition~\ref{prop:upper}, and Theorem~\ref{th:lower}.
\end{proof}

\section{Experimental Analysis}\label{sec:exp}

In this section, we present the experimental validation of our approach on two data sets (i)~synthetic data and (ii)~real-world maize data.

\subsection{Synthetic Data}
We considered a context set $\C$ and an action set $\X$ with  $5$-dimensional contexts and actions, i.e.,   $c \in \mathbb{R}^5$ and $x \in \mathbb{R}^5$,. Further we set the reward function $r(x_i, c_i)=\sum_{i=1}^5(x_i - c_i)^2$. Thus the parameterized vector $\phi(x,c)$ is given by $\phi () = [x_1^2, \ldots, x_5^2, c_1^2, \ldots, c_5^2, x_1c_1, \ldots, x_5c_5]$ and the  reward parameter $\theta^\star = [1, 1, 1, 1, 1, 1, 1, 1, 1, 1, -2, -2, -2, -2, -2]$.    The action set consists of $20$ actions that we sample from a standard Gaussian distribution. At each round $t \in T$, we sample the context $c_t$ from a multi-variate normal distribution and set the context distribution as  $\mu_t = \pazocal{N}(c_t, \mathbb{I}_5)$. The observation noise $\eta_t$ is set as Gaussian with zero mean and standard deviation $0.1$ and the  mean reward of the baseline policy at any time is taken to be the reward associated with the $10^{\mathrm{th}}$ best action.

\subsection{Maize Yield Data}
We use a maize yield data set acquired over four years by the maize Genomes to Fields (G2F) initiative \cite{mcfarland2020maize}, a multi-institutional effort in North America over 68 unique locations. The data set includes yields, planting dates, flowering times, and harvest dates, as well as hourly weather data from in-field weather stations, such as temperature, humidity, solar radiation, rainfall, and soil wind speed, as well as soil characteristics such as soil texture, organic matter, texture, and nitrogen, phosphorous, potassium, sulfur, and sodium levels (in parts per million). There are 2158 yield measurements (rewards) for 24 crops (action set) collected from 22 different locations in this data set. The weather data of the whole growing season was summarized by crop growth stages as in \cite{holzkamper2013identifying}.  These are average daily solar radiation [MJ/m2], average daily minimum temperature below ${0}^{\circ}$C in absolute values [${}^\circ$C] as a measure of frost impacts, average daily mean temperature [${}^{\circ}$C] as a measure of temperature determining plant growth, average daily maximum temperature above 35 C [${}^\circ$C] as a measure of heat stress, and average photoperiod. 

We first constructed a data set $\D=\{(c_i, x_i, y_i)\}$, where for each data point $i=1,2, \ldots, 2158$, the context $c_i \in \mathbb{R}^{28}$ is a $28-$dimensional vector that includes $6$-dimensional weather and soil data information (\% of sand, \% of silt, \% of clay in the soil, daily average temperature, radiation, and photosynthesis) and a $22-$dimensional one-hot encoding that captures the field  ID, and $x_i, y_i$ are the seed/crop identifier, yield, respectively. We first fit a bilinear model \cite{koren2009matrix} such that $y_i \approx c_i^{\top}WV_{x_i}$, where $V_{x_i} \in \mathbb{R}^{10}$ is the feature vector for crop type $x_i$ \cite{kirschner2019stochastic}. Our data set consists of $24$ varieties of crops and hence there are $24$ feature vectors, $V_1, V_2, \ldots, V_{24}$. The  bilinear model  captures the correlation between site features $c_i^{\top}W$ and $V_{x_i}$ for each data point and serves as the interactive setting that provides the rewards (yield) for our bandit setting. 

 We fitted a bilinear model on the historical maize data, collected through the G2F initiative \cite{mcfarland2020maize}, via stochastic gradient descent using the loss function  

\begin{equation*}\label{eq:loss}
    L(V,W) = \sum_{i=1}^{n} (y_{i} - c_{i}^{T}WV_{x_i})^2 + \lambda_{v} ||V_{x_i}||^2 + \lambda_{w} ||W||^2,
\end{equation*}

where  $\lambda_v$ and $\lambda_w$ denotes the regularization terms.  Training this model for 300 iterations resulted in a mean square error loss of 0.002 using a learning rate of $0.015$, $\lambda_v$ = $\lambda_w = 0.001$ and a latent dimension of $10$, i.e., $V_{x_i} \in \mathbb{R}^{10}$ for all $i$.  The observation noise $\eta_t$ is set as Gaussian with zero mean and standard deviation $0.1$ and the  mean reward of the baseline policy at any time is taken to be the reward associated with the $16^{\mathrm{th}}$ best action.

%\begin{figure}
%    \centering
%    \includegraphics[height=0.4\textwidth, width=0.5\textwidth]{download.png}
%    \caption{Visualization of yield values predicted by the model versus true yield values. The predicted values shows correlation with the true values.}\label{fig:model}
%\end{figure}

\subsection{Experiments and Analysis}
We performed two experiments on the synthetic and real data and all the points are averaged over 100 independent trials. In the first experiment we varied the value of the constraint parameter $\alpha$ and we plot the cumulative regret $\R_t$ at each round $t$.  Figure~\ref{fig:regret1} shows the comparison of the cumulative regret of the standard linear UCB algorithm (LUCB) in \cite{abbasi2011improved}, when contexts are known and there are no constraints, and the conservative stochastic bandit  with context distribution and $\alpha=0.1, 0.3, 0.5,$ and $0.8$ for the synthetic data.  
 Figure~\ref{fig:regret4} shows the comparison of the cumulative regret of the LUCB algorithm in \cite{abbasi2011improved} and the conservative stochastic bandit  with context distribution and $\alpha=0.2, 0.4, 0.6,$ and $0.9$ for the maize yield data data.
From Figures~\ref{fig:regret1} and~\ref{fig:regret4} we observe that as the value of $\alpha$ increases the cumulative regret decreases which is expected as larger value of $\alpha$ means weaker constraint.

In the second experiment, we implemented three cases of our bandit setting: (i)~when the learner observes the context and there are no safety/performance constraints (LUCB \cite{abbasi2011improved}), (ii)~when the learner observes the context and there are safety constraints (conservative linear UCB, CLUCB \cite{kazerouni2016conservative}), and (iii)~when the learner observes only the context distribution and there are safety constraints (conservative stochastic UCB algorithm with context distribution, Algorithm~\ref{alg:UCB}). Figures~\ref{fig:regret2} and~\ref{fig:regret5} presents the plots for the cumulative regret for the three settings for synthetic data and maize yield data, respectively. We note that, in (i) the decisions are based on the observed context, in (ii) the decisions are based on the observed context and the safety constraint and in (iii) the decisions are based on the context distribution and the safety constraint. From the plots we notice that (i) outperforms (ii) and (ii) outperforms (iii) which is expected. We also present the per step cumulative regret $\R_T/T$ for the synthetic data with $T=2000$ and fro the maize yield data with $T=1000$ while varying the value of $\alpha$.  From Figures~\ref{fig:regret3} and~\ref{fig:regret6} we notice that the gap between the standard linear UCB algorithm and the conservative stochastic bandit algorithm decreases as the $\alpha$ value increases, which is expected as larger $\alpha$ means weaker constraint.

\section{Conclusion}\label{sec:con}
In this paper, we presented a conservative stochastic contextual bandit framework for sequential decision making when  an adversary chooses a distribution on the set of possible contexts and the learner is subject to certain safety/performance constraints. Our bandit formulation is conservative in the sense that we incorporate constraints on the learned policy such that the learned policy  need to satisfy certain baseline performance criteria while maximizing the reward. Furthermore, our bandit formulation is  stochastic in the sense that the contexts are not observable, rather a distribution of the contexts are known.  We proposed a  conservative linear UCB algorithm for stochastic bandits with context distribution. We proved an upper bound on the regret of the algorithm and showed that it can be decomposed into three terms: (i)~an upper bound for the regret of the standard linear UCB algorithm, (ii)~a constant term (independent of time horizon)  that accounts for the loss of being conservative in order to satisfy the safety constraint, and (ii)~a constant term (independent of time horizon) that accounts for the loss of the contexts being unknown and only the distrbution being known.    We validated the performance of our approach through extensive simulations on synthetic data and on real-world  maize data collected through the Genomes to Fields (G2F) initiative.  
  
%\appendix

%\section{Appendix}

\bibliographystyle{myIEEEtran}
\bibliography{AgriCPS}
 
\end{document}